\documentclass[11pt]{article}
 
\usepackage{bm}
\usepackage{tablefootnote}
\usepackage{array}
\newcolumntype{C}[1]{>{\centering\let\newline\\\arraybackslash\hspace{0pt}}m{#1}}
\usepackage{relsize}
\usepackage{lscape}
\usepackage[round,longnamesfirst]{natbib}
\usepackage{needspace}
\usepackage{afterpage}
\usepackage{placeins}
\usepackage[left=1.3in, right=1.3in]{geometry}
\usepackage{soul}
\usepackage{threeparttable}
\usepackage{enumerate}
\usepackage[shortlabels]{enumitem}
\setlist[enumerate]{(i)}
\usepackage[bottom]{footmisc}
\usepackage{arydshln,float}
\usepackage{titlesec}
\titleformat{\section}
  {\scshape\bfseries\centering}{\thesection.}{1em}{}
\titleformat{\subsection}
  {\bfseries}{\thesubsection.}{1em}{}
\titleformat{\subsubsection}
  {\bfseries}{\thesubsubsection.}{1em}{}

\usepackage{nameref}
\usepackage{indentfirst}
\usepackage{amsmath}
\usepackage{amsthm}
\usepackage{amssymb}
\usepackage{nccmath}
\usepackage{booktabs}
\usepackage{color}
\usepackage{fancyhdr}
\usepackage[parfill]{parskip}
\usepackage{graphicx}

\usepackage{hypernat}
\definecolor{lust}{rgb}{0.9, 0.13, 0.13}
\definecolor{magenta(dye)}{rgb}{0.79, 0.08, 0.48}
\usepackage[
  bookmarks=true, 
  bookmarksopen=true, 
  breaklinks=true, 
  colorlinks=true,
  linkcolor=magenta(dye),
  citecolor=magenta(dye), 
]{hyperref}

\usepackage{mathrsfs}
\usepackage{caption}
\usepackage{subcaption}

\newtheoremstyle{assumption}%
	{\baselineskip}
	{\topsep}
	{\normalfont}
	{\parindent}
	{\bfseries}
	{.}
	{.5em}
	{}
\theoremstyle{assumption}
\newtheorem{assumption}{Assumption}
\newtheoremstyle{axiom}%
	{\baselineskip}
	{\topsep}
	{\normalfont}
	{\parindent}
	{\bfseries}
	{.}
	{.5em}
	{}
\theoremstyle{axiom}

\newtheoremstyle{claim}%
	{\baselineskip}
	{\topsep}
	{\itshape}
	{\parindent}
	{\scshape}
	{:}
	{.5em}
	{}
\theoremstyle{claim}

\newtheoremstyle{corollary}%
	{\baselineskip}
	{\topsep}
	{\normalfont}
	{\parindent}
	{\bfseries}
	{:}
	{.5em}
	{}
\theoremstyle{corollary}

\newtheorem{define}{Definition}
\newtheoremstyle{define}%
	{\baselineskip}
	{\topsep}
	{\normalfont}
	{\parindent}
	{\bfseries}
	{.}
	{.5em}
	{}
\theoremstyle{define}

\newtheoremstyle{example}%
	{\baselineskip}
	{\topsep}
	{\normalfont}
	{\parindent}
	{\scshape}
	{:}
	{.5em}
	{}
\theoremstyle{example}

\newtheoremstyle{lemma}%
	{\baselineskip}
	{\topsep}
	{\itshape}
	{\parindent}
	{\bfseries}
	{.}
	{.5em}
	{}
\theoremstyle{lemma}
\newtheorem{lemma}{Lemma}

\newtheoremstyle{theorem}%
	{\baselineskip}
	{\topsep}
	{\itshape}
	{\parindent}
	{\scshape \bfseries}
	{.}
	{.5em}
	{}
\theoremstyle{theorem}
\newtheorem{theorem}{Theorem}

\newtheorem{proposition}{Proposition}


\newcounter{parentnumber}

\newcommand{\argmin}{\operatornamewithlimits{argmin}}

\newcommand{\Var}{\mathrm{Var}}
\newcommand{\E}{\operatorname{\mathbb{E}}}
\newcommand{\pr}{\operatorname{Pr}}

\graphicspath{{./Figures/}}

\begin{document}

\title{{\normalfont Random Projection Estimation of Discrete-Choice Models
    with Large Choice Sets}\thanks{First draft: February 29, 2016. This draft: \today.   We thank  Hiroaki Kaido, Michael Leung, Sergio Montero, Harry Paarsch, Alejandro Robinson, Frank Wolak, and participants at the DATALEAD conference (Paris, November 2015) for helpful comments.}}
\author{Khai X. Chiong\thanks{USC Dornsife INET \& Department of Economics, University of Southern California. E-mail: \href{mailto:kchiong@usc.edu}{kchiong@usc.edu}} \and Matthew Shum\thanks{California Institute of Technology. E-mail: \href{mailto:mshum@caltech.edu}{mshum@caltech.edu}}}
\date{}

\maketitle

\begin{abstract}
We introduce {\em sparse random projection}, an important dimension-reduction tool from
machine learning, for the estimation of discrete-choice models with
high-dimensional choice sets. Initially,  high-dimensional data are
compressed into a lower-dimensional Euclidean space using random projections. Subsequently, estimation proceeds using cyclic monotonicity
moment inequalities implied by the multinomial choice model; the estimation
procedure is semi-parametric and does not require explicit distributional
assumptions to be made regarding the random utility errors. The random
projection procedure is justified via the Johnson-Lindenstrauss Lemma -- the pairwise distances between data points are preserved during data compression, which we exploit to show convergence of our estimator.  The estimator works well in simulations and in an application to a supermarket scanner
dataset.



\end{abstract}
\section{Introduction}

Estimation of discrete-choice models in which consumers face high-dimensional choice sets is computationally challenging.  In this paper, we propose a new estimator that is tractable for semiparametric multinomial models with very large choice sets. Our estimator utilizes {\em random projection}, a powerful dimensionality-reduction technique from the machine learning literature.   To our knowledge, this is the first use of random projection in the econometrics
literature on discrete-choice models.   Using random projection, we can feasibly estimate high-dimensional discrete-choice models without specifying particular distributions for the random utility errors -- our approach is semi-parametric. 
 
In random projection, vectors of high-dimensionality are replaced by random low-dimensional linear combinations of the components in the original vectors.   The Johnson-Lindenstrauss Lemma, the backbone of random projection techniques, justifies that with high probability, the high-dimensional vectors are embedded in a lower dimensional Euclidean space in the sense that pairwise distances and inner products among the projected-down lower-dimensional vectors are preserved.

Specifically, we are given a $d$-by-$l$ data matrix, where $d$ is the dimensionality of the choice sets. When $d$ is very large, we encounter computational problems that render estimation difficult: estimating semiparametric discrete-choice models is already challenging, but large choice sets exacerbate the computational challenges; moreover, in extreme cases, the choice sets may be so large that typical computers will not be able to hold the data in memory (RAM) all at once for computation and manipulation.\footnote{For example, \cite{ng2015opportunities} analyzes terabytes of scanner data that required an amount of RAM that was beyond the budget of most researchers.}

Using the idea of random projection, we propose first, in a data pre-processing step, pre-multiplying the large $d$-by-$l$ data matrix by
a $k$-by-$d$ (with $k<< d$) stochastic matrix, resulting in a smaller $k$-by-$l$ compressed data matrix that is more manageable. Subsequently, we estimate the discrete-choice model using the compressed data matrix, in place of the original high-dimensional dataset.  Specifically in the second step, we estimate the discrete-choice model without needing to specify the distribution of the random utility errors by using inequalities derived from {\em cyclic monotonicity}: -- a generalization of the notion of monotonicity for vector-valued functions which always holds for random-utility discrete-choice models; see (\cite{rockafellar1970convex}, \cite{chiong2016duality}.

A desirable and practical feature of our procedure is that the random projection matrix is sparse, so that generating and multiplying it with the large data matrix is computationally parsimonious. For instance, when the dimensionality of the choice set is $d=5,000$, the random projection matrix consists of roughly 99\% zeros, and indeed only 1\% of the data matrix is needed or sampled.


We show theoretically that the random projection estimator converges to the unprojected estimator, as $k$ grows large. We utilize results from the machine learning literature, which show that random projection enables embeddings of points from high-dimensional into low-dimensional Euclidean space with high probability, and hence we can consistently recover the original estimates from the compressed dataset. In the simulation, even with small and moderate $k$, we show that the noise introduced by random projection  is reasonably small.  In summary, $k$ controls the trade-off between using a small/tractable dataset for estimation, and error in estimation. 

As an application of our procedures, we estimate a model of soft drink choice in which households choose not only which soft drink product to purchase, but also the store that they shop at.   In the dataset, households can choose from over 3000 (store/soft drink product) combinations, and we use random projection to reduce the number of choices to 300, one-tenth of the original number.  




\subsection{Related Literature}

Difficulties in estimating multinomial choice models with very large choice sets were already considered in the earliest econometric papers on discrete-choice models (McFadden (\citeyear{mcfadden1974conditional}, \citeyear{mcfadden1978modelling})).   There, within the special multinomial logit case, McFadden discussed simulation approaches to estimation based on sampling the choices faced by consumers; subsequently, this ``sampled logit'' model was implemented in \cite{train1987demand}.  This sampling approach depends crucially on the multinomial logit assumption on the errors, and particularly on the independence of the errors across items in the large choice set.\footnote{See also \cite{davis2016segregated} and \cite{keane2012estimation} for other applications of sampled logit-type discrete choice models.  On a related note, \cite{gentzkow2016measuring} use a Poisson approximation to enable parallel computation of a multinomial logit model of legislators' choices among hundreds of thousands of phrases.}   

In contrast, the approach taken in this paper is semiparametric, as we avoid making specific parametric assumptions for the distribution of the errors.   Our closest antecedent is \cite{fox2007semiparametric}, who uses a maximum-score approach of Manski (\citeyear{manski1975maximum}, \citeyear{manski1985semiparametric}) to estimate semiparametric multinomial choice models with large choice sets but using only a subset of the choices.\footnote{\cite{fox2013measuring} use this estimator for a model of the FCC spectrum auctions, and also point out another reason whereby choice sets may be high-dimensionality: specifically, when choice sets of consumers consist of {\em bundles} of products. The size of this combinatorial choice set is necessarily exponentially increasing in the number of products. Even though the vectors of observed market shares will be sparse, with many zeros, as long as a particular bundle does not have zero market share across all markets, it will still contain identifying information.
} Identification  relies on a ``rank-order'' assumption, which  is an
implication of the Independence of Irrelevant Alternatives (IIA) property,
and hence can be considered as a generalized version of IIA.  It is
satisfied by exchangeability of the joint error
distribution. 

In contrast, our cyclic monotonicity approach allows for non-exchangeable
joint error distribution with arbitrary correlation between the
choice-specific error terms, but requires full independence of errors with
the observed covariates.\footnote{Besides \cite{fox2007semiparametric},
  the literature on semiparametric multinomial choice models is quite
  small, and includes the multiple-index approach of
  \cite{ichimura1991semiparametric} and \cite{lee1995semiparametric}, and
  a pairwise-differencing approach in \cite{powell2008simple}.   These
  approaches do not appear to scale up easily when choice sets are large,
  and also are not amenable to dimension-reduction using random
  projection.}  Particularly, our approach accommodates models with error
structures in the generalized extreme value family (ie. nested logit
models; which are typically non-exchangeable distributions), and we illustrate this in our empirical application below, where we consider a model of joint store and brand choice in which a nested-logit (generalized extreme value) model would typically be used.

Indeed, Fox's rank-order property and the cyclic monotonicity property used here represent two different (and non-nested) generalizations of Manski's (\citeyear{manski1975maximum}) maximum-score approach for semiparametric binary choice models to a multinomial setting.  The rank-order property restricts the dependence of the utility shocks across choices (exchangeability), while cyclic monotonicity restricts the dependence of the utility shocks across different markets (or choice scenarios).\footnote{\cite{haile2008empirical} refer to this independence of the utility shocks across choice scenarios as an ``invariance'' assumption, while \cite{goeree2005regular} call the rank-order property a ``monotonicity'' or ``responsiveness'' condition.}

The ideas of random projection were popularized in the Machine Learning literature on dimensionality reduction (\cite{vempala2000random,achlioptas2003database,dasgupta2003elementary}).  As these papers point out, both by mathematical derivations and computational simulations, random projection allows computationally simple and  low-distortion embeddings of points from high-dimensional into low-dimensional Euclidean space.  However, the random projection approach will not work with all high dimensional models.   The reason is that while the reduced-dimension vectors maintain the same length as the original vectors, the individual components of these lower-dimension matrices may have little relation to the components of the original vectors.   Thus, models in which the components of the vectors are important would not work with random projection.   

In many high-dimensional econometric models, however, only the lengths and inner products among the data vectors are important-- this includes least-squares regression models with a fixed number of regressors but a large number of observations and, as we will see here, aggregate (market-level) multinomial choice models where consumers in each market face a large number of choices.   But it will {\em not} work in, for instance, least squares regression models in which the number of observations are modest but the number of regressors is large -- such models call for regressor selection or reduction techniques, including LASSO or principal components.\footnote{See \cite{belloni2012sparse}, \cite{belloni2014high}, and \cite{gillen2015blp}.   Neither LASSO nor principal components do not maintain lengths and inner products of the data vectors; typically, they will result in reduced-dimension vectors with length strictly smaller than the original vectors.}

Section 2 presents our semiparametric discrete-choice modeling framework, and the moment inequalities derived from cyclic monotonicity which we will use for estimation.   In section 3, we introduce random projection and show how it can be applied to the semiparametric discrete-choice context to overcome the computational difficulties with large choice sets.   We also show formally that the random-projection version of our estimator converges to the full-sample estimator as the dimension of the projection increases.   Section 4 contains results from simulation examples, demonstrating that random projection works well in practice, even when choice sets are only moderately large.  In section 5, we estimate a model of households' joint decisions of store and brand choice, using store-level scanner data.   Section 6 concludes.

\section{Modeling Framework}
We consider a semiparametric multinomial choice framework in which the choice-specific utilities are assumed to take a single index form, but the distribution of utility shocks is unspecified and treated as a nuisance element.\footnote{
Virtually all the existing papers on semiparametric multinomial choices use similar setups (\cite{fox2007semiparametric}, \cite{ichimura1991semiparametric}, \cite{lee1995semiparametric}, \cite{powell2008simple}).
}   Specifically, an agent chooses from among $\mathcal{C}=[1,\dots, d]$ alternatives or
choices. High-dimensionality here refers to a large value of $d$. The
utility that the agent derives from choice $j$ is $\bm{X}_{j}\bm{\beta} +
\epsilon_{j}$, where $\bm{\beta}=(\beta_{1},\dots,\beta_{b})' \in
\mathbb{R}^{b}$ are unknown parameters, and $\bm{X}_{j}$ is a $1 \times
b$ vector of covariates specific to choice $j$. Here, $\epsilon_j$ is a utility shock, encompassing unobservables which affect the agent's utility from the $j$-th choice.

Let $u_{j}\equiv \bm{X}_j\bm{\beta}$ denote the deterministic part of utility that the agent derives from choice $j$, and let ${\bm u} = (u_{j})_{j=1}^{d}$, which we assume to lie in the set $\mathcal{U} \subseteq \mathbb{R}^{d}$. For a given ${\bm u} \in \mathcal{U}$,  the probability that the agent chooses $j$ is 
$p_{j}({\bm u}) = \Pr( u_{j} + \epsilon_{j} \geq \max_{k \neq j}\{u_{k}+ \epsilon_{k}\})$. Denote the vector of choice probabilities as ${\bm p}({\bm u}) = (p_{j}({\bm u}) )_{j=1}^{d}$. Now observe that the choice probabilities vector ${\bm p}$ is a vector-valued function such that  ${\bm p} : \mathcal{U} \rightarrow \mathbb{R}^{d}$. 

In this paper, we assume that the utility shocks $\bm{\epsilon}\equiv (\epsilon_1,\ldots, \epsilon_d)'$
are distributed independently of $\bm{X}\equiv(\bm{X}_1,\ldots, \bm{X}_d)$, but otherwise allow it to follow an unknown joint distribution that can be
arbitrarily correlated among different choices $j$.   This leads to the following proposition:

\begin{proposition}\label{prop:cm}
Let $\bm{\epsilon}$ be independent of $\bm{X}$.   Then the choice probability function ${\bm p} : \mathcal{U} \rightarrow \mathbb{R}^{d}$ satisfies  {\bf cyclic monotonicity}.
\end{proposition}

\begin{define}[Cyclic Monotonicity]
Consider a function ${\bm p} : \mathcal{U} \rightarrow \mathbb{R}^{d}$, where $\mathcal{U} \subseteq \mathbb{R}^{d}$. Take a length $L$-cycle of points in $\mathcal{U}$, denoted as the sequence $({\bm u}^{1},{\bm u}^{2},\dots,{\bm u}^{L},{\bm u}^{1})$. The function $\bm{p}$ is cyclic monotone with respect to the cycle $({\bm u}^{1},{\bm u}^{2},\dots,{\bm u}^{L},{\bm u}^{1})$ if and only if 
\begin{align}\label{cm}
\sum_{l=1}^{L} (\bm{u}^{l+1} - \bm{u}^{l}) \cdot {\bm p}(\bm{u}^{l}) \leq 0
\end{align}
where $\bm{u}^{L+1} = \bm{u}^{1}$. The function $\bm{p}$ is cyclic monotone on $\mathcal{U}$ if and only if it is cyclic monotone with respect to all possible cycles of all lengths on its domain (see \cite{rockafellar1970convex}).\hfill$\blacksquare$
\end{define}

Proposition \ref{prop:cm} arises from the underlying convexity properties of the discrete-choice problem.  We refer to \citet*{chiong2016duality} and \cite{shi2016estimating} for the full details.  Briefly, the independence of $\bm{\epsilon}$ and $\bm{X}$ implies that the {\em social surplus function} of the discrete choice model, defined as,
$$
\mathcal{G}({\bm u}) = \mathbb{E} \left[\max_{j\in\left\{1,\ldots, d\right\}} \left(u_j + \epsilon_j \right)\right]
$$
is convex in ${\bm u}$.   Subsequently, for each vector of utilities ${\bm
  u} \in \mathcal{U}$, the corresponding vector of choice probabilities
${\bm p}({\bm u})$, lies in the subgradient of $\mathcal{G}$ at ${\bm
  u}$;\footnote{See Theorem 1(i) in \cite{chiong2016duality}.  This is the
  Williams-Daly-Zachary Theorem (cf. \cite{mcfadden1981econometric}),
  generalized to the case when the social surplus function may be
  non-differentiable, corresponding to cases where the utility shocks
  $\bm{\epsilon}$ have bounded support or follow a discrete distribution.} that is:
\begin{equation}\label{subgradient}
{\bm p}({\bm u}) \in \partial \mathcal{G}({\bm u}).
\end{equation}

By a fundamental result in convex analysis (\cite{rockafellar1970convex}, Theorem 23.5), the subgradient of a convex function satisfies cyclic monotonicity, and hence satisfies the CM-inequalities in (\ref{cm}) above. (In fact, any function that satisfies cyclic monotonicity must be a subgradient of some convex function.) Therefore, cyclic monotonicity is the appropriate vector generalization of the fact that the slope of a scalar-valued convex function is monotone increasing. 

\subsection{Inequalities for Estimation}
Following \cite{shi2016estimating}, we use the cyclic monotonic inequalities in (\ref{cm}) to estimate the parameters $\bm{\beta}$.\footnote{See also \cite{melo2015testing} for an application of cyclic monotonicity for testing game-theoretic models of stochastic choice.} Suppose we observe the aggregate behavior of many independent agents across $n$ different markets.  In this paper, we assume the researcher has access to such aggregate data, in which the market-level choice probabilities (or market shares) are directly observed.   Such data structures arise often in aggregate demand models in empirical industrial organization (eg. \cite{berry2014identification}, \cite{gandhi2013estimating}).

Our dataset consists of $\mathcal{D}=\left((\bm{X}^{(1)},\bm{p}^{(1)}),\dots,(\bm{X}^{(n)},\bm{p}^{(n)})\right)$, $\bm{p}^{(i)}$ denotes the $d \times 1$ vector of choice probabilities, or market shares, in market $i$, and $\bm{X}^{(i)}$ is the $d \times b$ matrix of covariates for market $i$ (where row $j$ of $\bm{X}^{(i)}$ corresponds to $\bm{X}_{j}^{(i)}$,  the vector of covariates specific to choice $j$ in market $i$).  Assuming that the distribution of the utility shock  vectors $\left( \bm{\epsilon}^{(1)},\ldots,\bm{\epsilon}^{(n)}\right)$ is i.i.d. across all markets, then by Proposition \ref{prop:cm}, the cyclic monotonicity inequalities (\ref{cm}) will be satisfied across all cycles in the data $\mathcal{D}$: that is,
\begin{align}\label{cm2}
\sum_{l=1}^{L} (\bm{X}^{(a_{l+1})}\bm{\beta} - \bm{X}^{(a_{l})}\bm{\beta}) \cdot {\bm p}^{(a_{l})} \leq 0,\quad\text{for all cycles } \left(a_{l}\right)_{l=1}^{L+1} \text{ in data } \mathcal{D}, L\geq 2
\end{align}

Recall that a cycle in data $\mathcal{D}$ is a sequence of distinct integers $\left(a_{l}\right)_{l=1}^{L+1}$, where $a_{L+1} = a_{1}$, and each integer is smaller than or equal $n$, the number of markets.

From the cyclic monotonicity inequalities in (\ref{cm2}), we define a criterion function which we will optimize to obtain an estimator of $\bm\beta$. This criterion function is the sum of squared violations of the cyclic monotonicity inequalities:

\begin{align}
Q(\bm{\beta}) = \sum_{\text{all cycles in data } \mathcal{D}; L\geq 2}\Bigg[\sum_{l=1}^{L} \left(\bm{X}^{(a_{l+1})}\bm{\beta} - \bm{X}^{(a_{l})}\bm{\beta}\right) \cdot {\bm p}^{(a_{l})} \Bigg]_{+}^2 \label{Q}
\end{align}
where $[x]_{+} = \max\{x,0\}$.  Our estimator is defined as
\begin{align*}
\hat{\bm{\beta}} = \argmin_{\bm{\beta} \in \mathbb{B} : ||\bm{\beta}|| =1} Q(\bm{\beta}).
\end{align*}

The parameter space $ \mathbb{B}$ is defined to be a convex subset of $\mathbb{R}^{b}$. The parameters are normalized such that the vector $\hat{\bm \beta}$ has a Euclidean length of 1. This is a standard normalization that is also used in the Maximum Rank Correlation estimator, for instance, in \cite{han1987non} and \cite{hausman1998misclassification}. \cite{shi2016estimating} shows that the criterion function above delivers consistent interval estimates of the identified set of parameters under the assumption that the covariates are exogenous. The criterion function here is convex, and the global minimum can be found using subgradient descent (since it is not differentiable everywhere).\footnote{Because the cyclic monotonicity inequalities involve differences in $\bm{X}\bm{\beta}$, no constant terms need be included in the model, as it would simply difference out across markets.   Similarly, any outside good with mean utility normalized to zero would also drop out of the cyclic monotonicity inequalities.}

The derivation of our estimation approach for discrete-choice models does not imply that all the choice probabilities be strictly positive -- that is, zero choice probabilities are allowed for.\footnote{Specifically, Eq. (\ref{subgradient}) allows some of the components of the choice probability vector $\bm{p}(\bm{u})$ to be zero.}   The possibility of zero choice probabilities is especially important and empirically relevant especially in a setting with large choice sets, as dataset with large choice sets (such as store-level scanner data) often have zero choice probabilities for many products (cf. \cite{gandhi2013estimating}).

For reasons discussed earlier, high-dimensional choice sets posed particular challenges for semi-parametric estimation. Next, we describe how random projection can help reduce the dimensionality of our problem.

\section{Random Projection}\label{rp section}

Our approach consists of two-steps: in the first data-preprocessing step, the data matrix
$\mathcal{D}$ is embedded into a lower-dimensional Euclidean space. This
dimensionality reduction is achieved by premultiplying $\mathcal{D}$ with
a {\em random projection matrix}, resulting in a compressed data matrix
$\tilde{\mathcal{D}}$ with a fewer number of rows, but the same number of
columns (that is, the number of markets and covariates is not reduced, but
the dimensionality of choice sets is reduced). In the second step, the
estimator outlined in Equation (\ref{Q}) is computed using only the
compressed data $\tilde{\mathcal{D}}$. 

A random projection matrix $R$, is a $k$-by-$d$ matrix (with $k<< d$) such that each entry $R_{i,j}$ is
distributed i.i.d according to $\frac{1}{\sqrt{k}}F$, where $F$ is any
mean zero distribution.  For any $d$-dimensional vectors $\bm{u}$
and $\bm{v}$, premultiplication by $R$ yields the random
reduced-dimensional ($k\times 1$) vectors $R \bm{u}$ and $R\bm{v}$; thus, $R\bm{u}$ and
$R\bm{v}$ are the random projections of $\bm{u}$ and $\bm{v}$, respectively.

By construction, a random projection matrix $R$ has the property
that, given two high-dimensional vectors $\bm{u}$ and $\bm{v}$, the squared
Euclidean distance between the two projected-down vectors $\|R \bm{u}-R\bm{v}\|^{2}$ is a random variable with mean equal to $\|\bm{u}-\bm{v}\|^{2}$, the squared distance between the two original high-dimensional vectors.  Essentially, the random projection procedure replaces each high-dimensional vector $\bm{u}$ with a random lower-dimensional counterpart $\bm{\tilde{u}}=R\bm{u}$ the length of which is a mean-preserving spread of the original vector's length.\footnote{For a detailed discussion, see Chapter 1 in \cite{vempala2000random}.}

Most early applications of random projection utilized Gaussian random
projection matrices, in which each entry of $R$ is generated independently
from standard Gaussian (normal) distributions.   However, for computational
convenience and simplicity, we focus in this paper on {\bf sparse} random projection
matrices, in which many elements will be equal to zero with high
probability.  Moreover, different choice of probability distributions of $R_{i,j}$ can lead to different variance and error tail bounds of $\|R \bm{u}-R\bm{v}\|^{2}$. Following the work  of \cite{li2006very}, we introduce a class of sparse random projection matrices that can also be tailored to enhance the efficiency of random projection.

\begin{define}[Sparse Random Projection Matrix]\label{sparserp}
A sparse random projection matrix is a $k$-by-$d$ matrix $R$ such that each $i,j$-th entry is independently and identically distributed according to the following discrete distribution:
\begin{align*}
R_{i,j} = \sqrt{s} \begin{cases}
+1 & \text{with probability }\frac{1}{2s}\\
0 & \text{with probability }1-\frac{1}{s}\\
-1 & \text{with probability }\frac{1}{2s} \end{cases}\quad (s> 1).
\end{align*}
\end{define}

By choosing a higher $s$, we produce sparser random projection matrices. \cite{li2006very} show that:

\begin{align}
\Var(\|R \bm{u}-R\bm{v}\|^{2}) = \frac{1}{k} \Big(2\|\bm{u}-\bm{v}\|^{4}+(s-3)\sum_{j=1}^{d}(u_{j} - v_{j})^{4} \Big) \label{varrp}
\end{align}

It appears from this variance formula that higher value of $s$ reduces the efficiency of random projections. It turns out that when $d$ is large, which is exactly the setting where random projection is needed,  the first term in the variance formula above dominates the second term. Therefore, we can set  large values of $s$ to achieve very sparse random projection, with negligible loss in efficiency.  More concretely, we can set $s$ to be as large as $\sqrt{d}$. We will see in the simulation example that when $d=5,000$, setting $s=\sqrt{d}$ implies that the random projection matrix is zero with probability 0.986 -- that is, only 1.4\% of the data are sampled on average. Yet we find that sparse random projection performs just as well as a dense random projection.\footnote{More precisely, as shown by \cite{li2006very}, is that if all fourth moments of the data to be projected-down are finite, i.e. $\E[u^{4}_{j}]<\infty$, $\E[v^{4}_{j}]<\infty$, $\E[u^{2}_{j}v^{2}_{j}]<\infty$, for all $j=1,\dots,d$, then the term $\|\bm{u}-\bm{v}\|^{4}$ in the variance formula (Eq. \ref{varrp}) dominates the second term $(s-3)\sum_{j=1}^{d}(u_{j} - v_{j})^{4}$ for large $d$ (which is precisely the setting we wish to use random projection).} 

Besides the sparse random projection ($s= \sqrt{d}$), we will also try $s=1$, where the minimum variance is achieved. We call this the {\bf optimal} random projection. If we let $s=3$, we obtain a variance of $\frac{1}{k}2\|\bm{u}-\bm{v}\|^{4}$, which interestingly, is the same variance achieved by the 
benchmark Gaussian random projection (each
element of the random projection matrix is distributed i.i.d. according to the standard Gaussian, see \cite{achlioptas2003database}). Since Gaussian random projection is dense and has the same efficiency as the sparse random projection with $s=3$, the class of random projections proposed in Definition \ref{sparserp} is preferred in terms of both efficiency and sparsity. Moreover,  random uniform numbers are much easier to generate than Gaussian random numbers.

\subsection{Random Projection Estimator}

We introduce the random projection estimator. Given the dataset $\mathcal{D}=\{(\bm{X}^{(1)},\bm{p}^{(1)}),$ $\dots,(\bm{X}^{(n)},\bm{p}^{(n)})\}$, define the {\em compressed} dataset by $\tilde{\mathcal{D}}_{k}=\{(\tilde{\bm{X}}^{(1)},\tilde{\bm{p}}^{(1)}),\dots,(\tilde{\bm{X}}^{(n)},\tilde{\bm{p}}^{(n)})\}$, where $(\tilde{\bm{X}}^{(i)},\tilde{\bm{p}}^{(i)}) = (R\bm{X}^{(i)},R\bm{p}^{(i)})$ for all markets $i$, and $R$ being a sparse $k \times d$  random projection matrix as in Definition \ref{sparserp}. 

\begin{define}[Random projection estimator]\label{rpestimator}
The random projection estimator is defined as $\tilde{\bm \beta}_{k} \in \argmin_{\bm \beta} Q(\bm{\beta},\tilde{\mathcal{D}}_{k})$, where $Q(\bm{\beta},\tilde{\mathcal{D}}_{k})$ is the criterion function in Equation (\ref{Q}) in which the input data is $\tilde{\mathcal{D}}_{k}$.\hfill$\blacksquare$
\end{define}

The compressed dataset $\tilde{\mathcal{D}}_{k}$ has $k$ number of rows, where the original dataset has a larger number of rows, $d$. Note that the identities of the markets and covariates (i.e. the columns of the data matrix) are unchanged in the reduced-dimension data matrix; as a result, the same compressed dataset can be used to estimate different utility/model specifications with varying combination of covariates and markets.

We will benchmark the random projection estimator with the estimator $\hat{{\bm \beta}} \in \argmin Q(\bm{\beta},\mathcal{D})$, where $Q(\bm{\beta},\mathcal{D})$ is the criterion function in Equation (\ref{Q}) in which the uncompressed data $\mathcal{D}$ is used as input. In the next section, we will prove convergence of the random projection estimator to the benchmark estimator using uncompressed data, as $k$ grows large. Here we provide some intuition and state some preliminary results for this convergence result.

  Recall from the previous section that the Euclidean distance between two vectors are preserved in expectation as these vectors are compressed into a lower-dimensional Euclidean space. In order to exploit this feature of random projection for our estimator, we rewrite the estimating inequalities -- based on cyclic monotonicity -- in terms of Euclidean norms. 
 
 \begin{define}[Cyclic Monotonicity in terms of Euclidean norms]\label{cm_euclid}
 Consider a function ${\bm p} : \mathcal{U} \rightarrow \mathbb{R}^{d}$, where $\mathcal{U} \subseteq \mathbb{R}^{d}$. Take a length $L$-cycle of points in $\mathcal{U}$, denoted as the sequence $({\bm u}^{1},{\bm u}^{2},\dots,{\bm u}^{L},{\bm u}^{1})$. The function $\bm{p}$ is cyclic monotone with respect to the cycle ${\bm u}^{1},{\bm u}^{2},\dots,{\bm u}^{L},{\bm u}^{1}$ if and only if 
 
 \begin{align}\label{cm_2}
 \sum_{l=2}^{L+1} \left(\|\bm{u}^{l} - \bm{p}^{l}\|^{2} -\| \bm{u}^{l} - {\bm p}^{l-1}\|^{2} \right)\leq 0
 \end{align}
where $\bm{u}_{L+1} = \bm{u}_{1}$, and $\bm{p}^l$ denotes $\bm{p}(\bm{u}^{l})$. The function $\bm{p}$ is cyclic monotone on $\mathcal{U}$ if and only if it is cyclic monotone with respect to all possible cycles of all lengths on its domain.\hfill$\blacksquare$
\end{define}

The inequalities (\ref{cm}) and (\ref{cm_2})  equivalently defined cyclic monotonicity, a proof is given in the appendix.  Therefore, from Definition \ref{cm_euclid}, we can rewrite the estimator in (\ref{Q}) as $\hat{\bm \beta} = \argmin_{\bm{\beta} \in  \mathbb{B}}Q(\bm{\beta})$ where the criterion function is defined as the sum of squared violations of the cyclic monotonicity inequalities:

\begin{align}
Q(\bm{\beta}) = \sum_{\text{all cycles in data } \mathcal{D}; L\geq 2}\Bigg[\sum_{m=2}^{L+1} \Big(\|\bm{X}^{(a_{l})}\bm{\beta} -{\bm p}^{(a_{l})}\|^{2} - \| \bm{X}^{(a_{l})}\bm{\beta} -{\bm p}^{(a_{l-1})} \|^{2}\Big)\Bigg]^2_{+}\label{QEuclidean}
\end{align}

To see the intuition behind the random projection estimator, we introduce the Johnson-Lindenstrauss Lemma. This lemma states that there exists a linear map (which can be found by drawing different random projection matrices) such that there is a low-distortion embedding.   There are different versions of this theorem; we state a typical one:

 \begin{lemma}[Johnson-Lindenstrauss]
 Let $\delta \in (0,\frac{1}{2})$. Let $\mathcal{U} \subset \mathbb{R}^{d}$ be a set of $C$ points, and $k = O(\log C/\delta^2)$. There exists a linear map ${\bm f} : \mathbb{R}^{d} \rightarrow \mathbb{R}^{k}$ such that for all ${\bm u},{\bm v} \in \mathcal{U}$:
 \begin{align*}
 (1-\delta) \|{\bm u}-{\bm v}\|^{2} \leq \|{\bm f} ({\bm u}) - {\bm f} ({\bm v}) \|^{2} \leq  (1+\delta) \|{\bm u}-{\bm v}\|^{2}.
 \end{align*}
 \end{lemma}

Proofs of the Johnson-Lindenstrauss Lemma can be found in, among others, \cite{dasgupta2003elementary,achlioptas2003database,vempala2000random}.   The proof is probabilistic, and demonstrates that, with a non-zero probability, the choice of a random projection ${\bm f}$ satisfies the error bounds stated in the Lemma. For this reason, the Johnson-Lindenstrauss Lemma has become a term that collectively represents random projection methods, even when the implication of the lemma is not directly used. 

As the statement of the Lemma makes clear, the reduced-dimension $k$ controls the trade-off between tractability and error in estimation. Notably, {\em these results do not depend on $d$}, the original dimension of the choice set (which is also the number of columns of $R$.)  Intuitively this is because the JL Lemma only requires that the lengths are maintained between the set of projected and unprojected vectors.   The definition of the random projection matrix (recall section~\ref{rp section} above) ensures that the length of each projected vector is an unbiased estimator of the length of the corresponding unprojected vector, regardless of $d$; hence, $d$ plays no direct role in satisfying the error bounds postulated in the JL Lemma.\footnote{However, $d$ does affect the variance of the length of the projected vectors, and hence affects the probabilities of achieving those bounds; see \cite{achlioptas2003database} for additional discussion.}

According to \cite{li2006very},``the JL lemma is conservative in many applications because it was derived based on Bonferroni correction for multiple comparisons."   That is, the magnitude for $k$ in the statement of the Lemma is a worst-case scenario, and larger than necessary in many applications.   This is seen in our computational simulations below, where we find that small values for $k$ still produce good results.

The feature that the cyclic monotonicity inequalities can be written in terms of Euclidean norms between vectors justifies the application of the Johnson-Lindenstrauss Lemma, and hence random projection, to our estimator, which is based on these inequalities.   In contrast, the ``rank-order'' inequalities, which underlie the maximum score approach to semiparametric multinomial choice estimation,\footnote{For instance,\cite{manski1985semiparametric}, \cite{fox2007semiparametric}.  The rank-order property makes pairwise comparisons of choices {\em within} a given choice set, and state that, for all $i,j\in \mathcal{C}$, $p_i(\bm{u})> p_j(\bm{u})$ iff $u_i>u_j$.} cannot be rewritten in terms in terms of Euclidean norms between data vectors, and hence random projection cannot be used for those inequalities.
 
%
%
%

%
%
%

\subsection{Convergence}\label{convsection}

In this section we show that, for any given data $\mathcal{D}$, the random
projection estimator computed using the compressed data
$\tilde{\mathcal{D}}_{k} = R \cdot \mathcal{D}$ converges in probability
to the corresponding estimator computed using the uncompressed data
$\mathcal{D}$, as $k$ grows large, where $k$ is the number of rows in the
random projection matrix $R$. We begin with simplest case where the
dimensionality of the original choice set $d$ is fixed, while the
reduced-dimension $k$ grows.\footnote{In Appendix \ref{add} we consider the case where $d$ grows with $k$.}

  In order to highlight the random projection aspect of our estimator, we assume that the market shares and other data variables are observed without error.  Hence, given the original (uncompressed) data $\mathcal{D}$, the criterion function $Q({\bm \beta},\mathcal{D})$ is deterministic, while the criterion function $Q({\bm \beta},\tilde{\mathcal{D}}_{k})$ is random solely due to the random projection procedure. 

All proofs for results in this section are provided in Appendix C.  We first show that the random-projected criterion function converges uniformly to the unprojected criterion function:

\begin{theorem}[Uniform convergence of criterion function]\label{ucon}
For any given dataset $\mathcal{D}$, we have $\sup_{{\bm \beta} \in \mathbb{B}} |Q({\bm \beta}, \tilde{\mathcal{D}}_{k}) - Q({\bm \beta},\mathcal{D})|  \xrightarrow[]{p} 0$, as $k$ grows. 
\end{theorem}

Essentially, from the defining features of the random projection matrix $R$, we can argue that $Q({\bm\beta} ,\tilde{\mathcal{D}}_{k})$ converges in probability to $Q({\bm\beta}, \mathcal{D})$, {\em pointwise} in $\bm\beta$.   Then, because $Q({\bm\beta},\mathcal{D})$ is convex in $\bm\beta$ (which we will also show), we can invoke the Convexity Lemma from \cite{pollard1991asymptotics}, which says that pointwise and uniform convergence are equivalent for convex random functions.

%

Finally, under the assumption that the deterministic criterion function $Q(\bm{\beta},\mathcal{D})$ (i.e. computed without random projection) admits an identified set, then the random projection estimator converges in a set-wise sense to the same identified set. Convergence of the set estimator here means convergence in the {\em Hausdorff} distance, where the Hausdorff distance is a distance measure between two sets is: $
d(X,Y) = \sup_{y \in Y} \inf_{x\in X} \|x - y \| + \sup_{x \in X} \inf_{y \in Y} \|x -y \|$.

\begin{assumption}[Existence of identified set $\Theta^{*}$] \label{exist}
For any given data $\mathcal{D}$, we assume that there exists a set $\Theta^{*}$ (that depends on $\mathcal{D}$) such that $\sup_{{\bm \beta} \in \Theta^{*}} Q({\bm \beta},\mathcal{D}) = \inf_{{\bm \beta} \in \Theta^{*}} Q({\bm \beta},\mathcal{D})$ and $\forall \nu > 0$, $\inf_{{\bm \beta} \notin B(\Theta^{*},\nu)} Q({\bm \beta},\mathcal{D}) > \sup_{{\bm \beta} \in \Theta^{*}} Q({\bm \beta},\mathcal{D})$, where $B(\Theta^{*},\nu)$ denotes a union of open balls of radius $\nu$ each centered on each element of $\Theta^{*}$.
\end{assumption}

\begin{theorem}\label{convergence}
Suppose that Assumption \ref{exist} hold. For any given data $\mathcal{D}$, the random projection estimator $\tilde{\Theta}_{k} =\argmin_{{\bm \beta} \in \mathbb{B}} Q({\bm \beta},\tilde{\mathcal{D}}_{k})$ converges in half-Hausdorff distance to the identified set $\Theta^{*}$ as $k$ grows, i.e. $\sup_{{\bm \beta} \in \tilde{\Theta}_{k}} \inf_{{\bm \beta}' \in \Theta^{*}} \|{\bm \beta}  - {\bm \beta}'\|  \xrightarrow[]{p} 0$ as $k$ grows.
\end{theorem}

In the Appendix \ref{add}, we analyse the setting where the dimensionality of the choice set, $d$, grows with $k$, with $d$ growing much faster than $k$. Specifically, we let $d = O(k^{2})$ and show that convergence still holds true under one mild assumption. This assumption says that for all $d$-dimensional vectors of covariates $\bm{X}$ in the data $\mathcal{D}$, the fourth moment $\frac{1}{d}\sum_{j}^{d} (X_{j})^{4}$ exists as $d$ grows.


\section{Simulation Examples}

In this section, we show simulation evidence that random projection performs well in practice.
In these simulations, the sole source of randomness is the random projection matrices. This allows us to starkly examine the noise introduced by random projections, and how the performance of random projections varies as we change $k$, the reduced dimensionality. Therefore the market shares and other data variables are assumed to be observed without error.  

The main conclusion from this section is that the error introduced by random projection is negligible, even when the reduced dimension $k$ is very small. In the tables below, we see that the random projection method produces interval estimates that are always strictly nested within the identified set which was obtained when the full uncompressed data are used. 


\subsection{Setup}

We consider projecting down from $d$ to $k$. Recall that $d$ is the number of choices in our context. There are $n=30$ markets. The utility that an agent in market $m$ receives from choice $j$ is $U^{(m)}_{j} = \beta_{1}X^{(m)}_{1,j} + \beta_{2}X^{(m)}_{2,j}$, where $X^{(m)}_{1,j}\sim N(1,1)$ and $X^{(m)}_{2,j}\sim N(-1,1)$ independently across all choices $j$ and markets $m$.\footnote{We also considered  two other sampling assumptions on the regressors, and found that the results are robust to: (i) strong brand effects: $X^{(m)}_{l,j}=X_{l,j}+\eta^{(m)}_{l,j},\ l=1,2$, where $X_{1,j} \sim N(1,0.5)$, $X_{2,j} \sim N(-1,0.5)$, and $\eta^{(m)}_{l,j} \sim N(0,1)$; (ii)  strong market effects: $X^{(m)}_{l,j}=X^{(m)}_{l}+\eta^{(m)}_{l,j},\ l=1,2$, where $X^{(m)}_{1} \sim N(1,0.5)$, $X^{(m)}_{2} \sim N(-1,0.5)$, and $\eta^{(m)}_{l,j} \sim N(0,1)$.}

 We normalize the parameters $\beta = (\beta_{1},\beta_{2})$ such that $\|\beta\|=1$. This is achieved by parameterizing $\beta$ using polar coordinates: $\beta_1=\cos\theta$ and $\beta_2=\sin\theta$, where  $\theta\in [0,2\pi]$. The true parameter is $\theta_0=0.75\pi = 2.3562$.  

To highlight a distinct advantage of our approach, we choose a distribution of the error term that is neither exchangeable nor belongs to the generalized extreme value family. Specifically, we let the additive error term be a MA(2) distribution where errors are serial correlated in errors across products. To summarize, the utility that agent in market $m$ derives from choice $j$ is $U^{(m)}_{j} + \epsilon^{(m)}_{j}$,  where $\epsilon^{(m)}_{j} = \frac{1}{3}\sum_{l=0}^{3} \eta^{(m)}_{j+l}$, and $\eta^{(m)}_{j}$ is distributed i.i.d with $N(0,1)$.

Using the above specification, we generate the data $\mathcal{D}=\{(\bm{X}^{(1)},\bm{p}^{(1)}),\dots,(\bm{X}^{(n)},\bm{p}^{(n)})\}$ for $n=30$ markets, where $\bm{p}^{(m)}$ corresponds to the $d$-by-1 vector of simulated choice probabilities for market $m$: the $j$-th row of $\bm{p}^{(m)}$ is  $\bm{p}^{(m)}_{j} =\pr\left(U^{(m)}_{j} + \epsilon^{(m)}_{j} > U^{(m)}_{-j} + \epsilon^{(m)}_{-j}\right)$. We then perform random projection on $\mathcal{D}$ to obtain the compressed dataset $\tilde{\mathcal{D}}=\{(\tilde{\bm{X}}^{(1)},\tilde{\bm{p}}^{(1)}),\dots,(\tilde{\bm{X}}^{(n)},\tilde{\bm{p}}^{(n)})\}$. Specifically, for all markets $m$, $(\tilde{\bm{X}}^{(m)},\tilde{\bm{p}}^{(m)}) = (R\bm{X}^{(i)},R\bm{p}^{(i)})$, where $R$ is a realized $k \times d$ random projection matrix as in Definition \ref{sparserp}. Having constructed the compressed dataset, the criterion function in Eq. \ref{Q} is used to estimate $\beta$. We restrict to cycles of length 2 and 3 in computing Eq. \ref{Q}; however, we find that even using cycles of length 2 did not change the result in any noticeable way.

The random projection matrix is parameterized by $s$ (see Definition \ref{sparserp}). We set $s=1$, which  corresponds to the optimal random projection matrix. In Table \ref{sparseiid}, we show that sparse random projections ($s=\sqrt{d}$ in Definition \ref{sparserp}) perform just as well. Sparse random projections are much faster to perform -- for instance when $d=5000$, we sample less than 2\% of the data, as over 98\% of the random projection matrix are zeros.

In these tables, the rows correspond to different designs where the dimension of the dataset is projected down from $d$ to $k$. For each design, we estimate the model using 100 independent realizations of the random projection matrix. We report the means of the upper and lower bounds of the estimates, as well as their standard deviations. We also report the interval spans by the 25th percentile of the lower bounds as well as the 75th percentile of the upper bounds. The last column reports the actual identified set that is computed without using random projections. (In the Appendix, Tables \ref{iidapp} and \ref{sparseapp}, we see that in all the runs,  our approach produces interval estimates that are strictly nested within the identified sets.) 

The results indicate that, in most cases, optimization of the randomly-projected criterion function $Q({\bm\beta}, {\mathcal{D}_{k}})$ yields a unique minimum, in contrast to the unprojected criterion function $Q({\bm\beta}, {\mathcal{D}})$, which is minimized at an interval.  For instance, in the fourth row of Table \ref{iidlarged} (when compressing from $d=5000$ to $k=100$), we see that the true identified set for this specification, computed using the unprojected data, is $[1.2038, 3.5914]$, but the projected criterion function is always uniquely minimized (across all 100 replications). Moreover the average point estimate for $\theta$ is equal to 2.3766, where the true value is 2.3562.   This is unsurprising, and occurs often in the moment inequality literature; the random projection procedure introduces noise into the projected inequalities so that, apparently, there are no values of the parameters $\bm\beta$ which jointly satisfy all the projected inequalities, leading to a unique minimizer for the projected criterion function.   

\begin{table}[H]
  \begin{center}
\caption{Random projection estimator with optimal random projections, $s=1$}\label{mcunsparse}

\begin{tabular}{|c||ccc|c|}
\hline
Design  & mean LB (s.d.)  & mean UB (s.d.)   &{\small 25th LB, 75th UB}   &  True id set \\\hline\hline
$d=100, k=10$   &  2.3459   (0.2417) & 2.3459  (0.2417) & [2.1777, 2.5076]   & [1.4237, 3.2144]   \\

$d=500, k=100$   &  2.2701 (0.2582)  &  2.3714 (0.2832) & [2.1306, 2.6018]  & [1.2352, 3.4343] \\

$d=1000, k=100$   & 2.4001 (0.2824) &  2.4001 (0.2824) & [2.2248, 2.6018]  &[1.1410, 3.4972]  \\ 

$d=5000, k=100$ & 2.3766  (0.3054) & 2.3766 (0.3054)  & [2.1306, 2.6018] & [1.2038, 3.5914]  \\
  
$d=5000, k=500$   & 2.2262 (0.3295) &   2.4906 (0.3439) &  [1.9892, 2.7667] & [1.2038, 3.5914]  \\

%
 
\hline
\end{tabular}

{\small Replicated 100 times using independently realized random projection matrices. The true value of $\theta$ is 2.3562. Right-most column reports the interval of points that minimized the unprojected criterion function.}\label{iidlarged}
  \end{center}
  \end{table}
   
  \begin{table}[H]
  \begin{center}
\caption{Random projection estimator with sparse random projections, $s = \sqrt{d}$}\label{mcsparse}

\begin{tabular}{|c||ccc|c|}
\hline
Design  & mean LB (s.d.)   & mean UB (s.d.)  & {\small 25th LB, 75th UB} & True id set  \\\hline\hline
$d=100, k=10$   &  2.3073  (0.2785) & 2.3073 (0.2785) & [2.1306, 2.5076] & [1.4237, 3.2144]   \\

$d=500, k=100$   &  2.2545   (0.2457) & 2.3473 (0.2415) & [2.0363, 2.5076]  & [1.2352, 3.4343]   \\

$d=1000, k=100$   & 2.3332 (0.2530) &  2.3398 (0.2574) & [2.1777, 2.5076]  & [1.1410, 3.4972]  \\ 

$d=5000, k=100$ & 2.3671  (0.3144) &  2.3671 (0.3144) & [2.1777, 2.5547]& [1.2038, 3.5914]  \\
  
$d=5000, k=500$   & 2.3228  (0.3353) &  2.5335 (0.3119) & [2.1306, 2.7667]& [1.2038, 3.5914]    \\

%
 
\hline
\end{tabular}

{\small Replicated 100 times using independently realized {\bf sparse} random projection matrices (where $s = \sqrt{d}$ in Definition \ref{sparserp}). The true value of $\theta$ is 2.3562. Right-most column reports the interval of points that minimized the unprojected criterion function.}\label{sparseiid}
  \end{center}
  \end{table}

 \FloatBarrier

\section{Empirical Application: a discrete-choice model incorporating both store and brand choices}

For our empirical application, we use supermarket scanner data made available by the Chicago-area Dominicks supermarket chain.\footnote{This
  dataset has previously been used in many papers in both economics
  and marketing; see a partial list at {\sffamily http://research.chicagobooth.edu/kilts/marketing-databases/dominicks/papers}.} Dominick's operated a chain of grocery stores across the Chicago
 area, and the database recorded sales information on many product categories, at the store and week level, at each Dominick's store. For this application, we look at the soft drinks category. 

For our choice model, we consider a model in which consumers choose both
the type of soft drink, as well as the store at which they make their
purchase.   Such a model of joint store and brand choice allows consumers
not only to change their brand choices, but also their store choices, in
response to across-time variation in economic conditions.  For instance, 
\cite{coibion2015cyclicality} is an analysis of supermarket scanner
data which suggests the importance of ``store-switching'' in dampening the
effects of inflation in posted store prices during recessions.   

Such a model of store and brand choice also highlights a key benefit of our
semiparametric approach.   A typical parametric model which would be used to model
store and brand choice would be a nested logit model, in which the
available brands and stores would belong to different tiers of nesting
structure.   However, one issue with the nested logit approach is that the
results may not be robust to different
assumptions on the nesting structure-- for instance, one researcher may 
nest brands below stores, while another researcher may be inclined to nest stores below brands.  These two alternative specifications would
differ in how the joint distribution of the utility shocks between brands
at different stores are modeled, leading to different parameter estimates. Typically, there are no {\em a priori} guides on the correct nesting structure to impose.\footnote{Because of this, \cite{hausman1984specification} have developed formal
econometric specification tests for the nested logit model.}

In this context, a benefit of our semiparametric is that we are {\em
  agnostic} as to the joint distribution of utility shocks; hence our
approach accommodates both models in which stores are in the upper nest and
brands in the lower nest, or vice versa, or any other model in which the
stores or brands could be divided into further sub-nests.   

\begin{figure}[!h]
\centering
\includegraphics[scale=0.4]{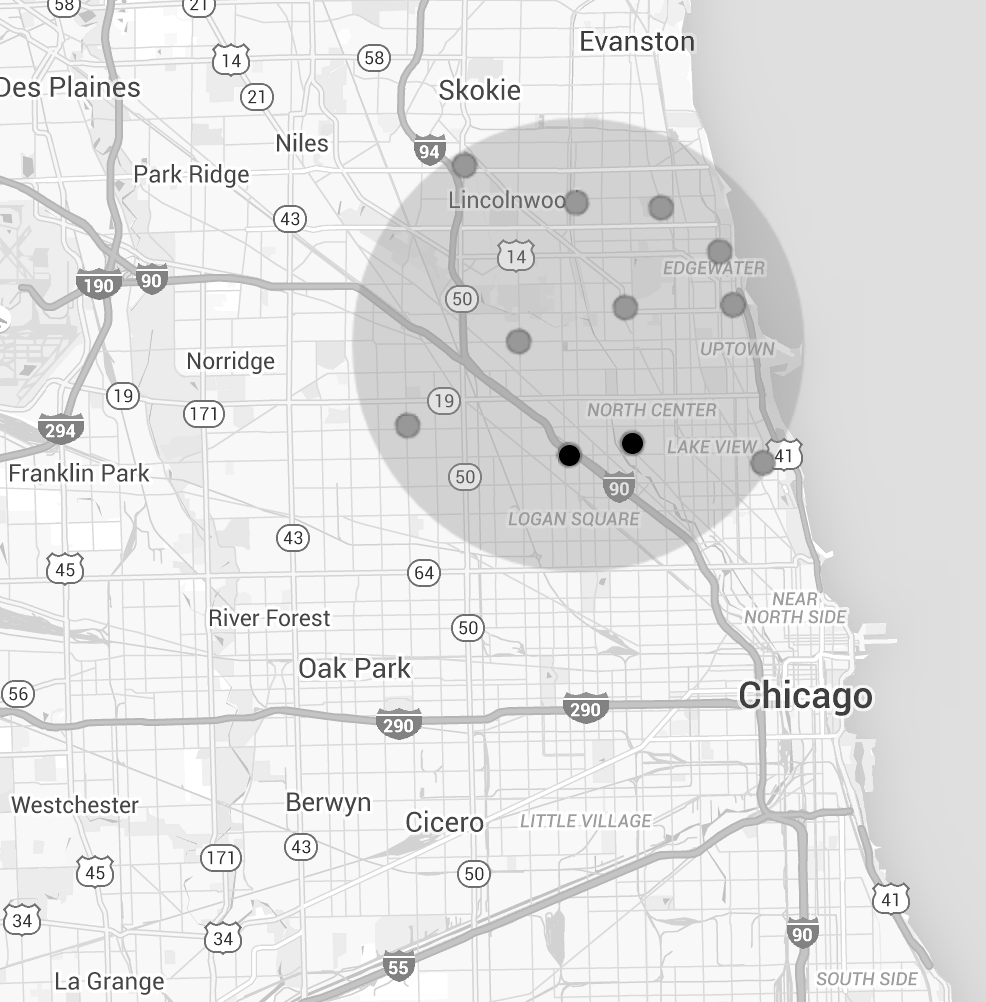}
\caption{Location of the 11 Dominick's branches as indicated by spots.}\label{11stores} 

{\footnotesize Radius of the circle is 4 miles. The darker spots are Dominick's medium-tier stores, the rest are high-tiers.}
\end{figure}

We have $n=15$ ``markets'', where each market corresponds to a distinct
two-weeks interval between October 3rd 1996 to April 30th 1997, which is the last recorded date.  
We include sales at eleven Dominicks supermarkets in north-central Chicago, as illustrated in Figure~\ref{11stores}.  Among these eleven supermarkets, most are classified as premium-tier stores, while two are medium-tier stores (distinguished by dark black spots in Figure~\ref{11stores}); stores in different tiers sell different ranges of products.

\begin{table}[!h]
\begin{center}
\small
\begin{tabular}{|c|C{7.5cm} |c|c|}
\hline
  & Definition  & Summary statistics  \\\hline\hline
 $price_{ij}$   & The average price of the store-upc $j$ at market $i$ & Mean: \$2.09, s.d: \$1.77    \\ \hline
 $bonus_{ij}$   & The fraction of weeks in market $i$ for which store-upc $j$ was on sale as a bonus or promotional purchase; for instance ``buy-one-get-one-half-off" deals & Mean: 0.27, s.d: 0.58  \\\hline
 $quantity_{ij}$&  total units of store-upc $j$ sold in market $i$ & Mean: 60.82, s.d: 188.37    \\\hline
 $holiday_{ij}$ & A dummy variable indicating the period spanning 11/14/96 to 12/25/96, which includes the Thanksgiving and Christmas holidays   &  6 weeks (3 markets)  \\\hline
$medium\_tier_{ij}$ & Medium, non-premium stores.${}^{a}$ & 2 out of 11 stores       \\\hline
$d$ & Number of store-upc & 3059\\\hline
\end{tabular}
\caption{Summary statistics}\label{summary}

{\footnotesize Number of observations is 45885=3059 upcs $\times$ 15 markets (2-week periods).}\label{summary}

{\footnotesize ${}^{a}$: Stores in the same tier share similar product selection, and also pricing to a certain extent.}
\end{center}
\end{table}

Our store and brand choice model consists of $d=3059$ choices, each corresponds to a unique store and universal product code (UPC) combination. We also define an outside option, for a total of $d=3060$ choices.\footnote{The outside option is constructed as follows: first we construct the market share $p_{ij}$ as $p_{ij} = quantity_{\.ij}/custcoun_{i}$, where $quantity_{ij}$ is the total units of store-upc $j$ sold in market $i$, and $custcoun_{i}$ is the total number of customers visiting the 11 stores and purchasing something at market $i$. The market share for market $i$'s outside option is then $1-\sum_{j=1}^{3093}p_{ij}$.}

%
 The summary statistics for our data sample are in Table \ref{summary}.

%

\begin{table}[!h]
\footnotesize
\begin{center}
\begin{tabular}{|c||c|c|c|c|}
\hline
 Specification & (A) & (B) & (C)   & (D) \\
 \hline
price & $-0.6982$ & $-0.9509$ &  $-0.7729$      & $ -0.4440$ \\               
 & $[-0.9420, -0.3131]$ & $[ -0.9869,  -0.7874]$ & $[ -0.9429,-0.4966]$  &$[-0.6821,-0.2445]$\\
bonus & &$0.0580$  &  $0.0461$    & $0.0336$  \\ 
& &$[-0.0116, 0.1949 ]$ & $[0.0054,0.1372]$ & $[0.0008, 0.0733]$  \\
 
price $\times$ bonus  & & $-0.1447$ &  $-0.0904$   & $ -0.0633 $\\
 & & $[-0.4843, 0.1123]$ & $[-0.3164 , 0.0521]$  & $[-0.1816,0.0375]$ \\
holiday &  0.0901  & &  $0.0661$   & 0.0238\\
 & $[-0.0080, 0.2175]$ & & $[-0.0288, 0.1378]$   & $[-0.0111, 0.0765]$\\
price $\times$ holiday &  $-0.6144$  & &$-0.3609$   &$-0.1183$\\
 & $[ -0.9013,  -0.1027]$ & & $[ -0.7048, -0.0139]$  &$[ -0.2368,  -0.0164]$ \\
  price $\times$ medium\_tier & & &   &$0.4815$\\ 
 & &  &   &$[-0.6978,  0.8067]$ \\\hline
   &  \multicolumn{4}{|c|}{$k=300$} \\
   &  \multicolumn{4}{|c|}{Cycles of length 2 \& 3} \\\hline
\end{tabular}
\caption{Random projection estimates, dimensionality reduction from $d=3059$ to $k=300$.}\label{p3}

{\footnotesize First row in each entry present the median coefficient, across 100 random projections.   Second row presents the 25-th and 75-th percentile among the 100 random projections. We use cycles of length 2 and 3 in computing the criterion function (Eq. \ref{Q}).}
\end{center}
\end{table}



Table \ref{p3} presents the estimation results.   As in the simulation results above, we ran 100 independent random projections, and thus obtained 100 sets of parameter estimates, for each model specification.   The results reported in Table \ref{p3} are therefore summary statistics of the estimates for each parameter. Since no location normalization is imposed for the error terms, we do not include constants in any of the specifications.  For estimation, we used cycles of length of length 2 and 3.\footnote{The result did not change in any noticeable when we vary the length of the cycles used in estimation.}

Across all specifications, the price coefficient is strongly negative. 
The {\em holiday} indicator has a positive (but small) coefficient, suggesting that, all else equal, the end-of-year holidays are a period of peak demand for soft drink products.\footnote{cf. \cite{chevalier2003don}.  Results are similar if we define the holiday period to extend into January, or to exclude Thanksgiving.}   In addition, the interaction between {\em price} and {\em holiday} is strongly negative across specifications, indicating that households are more price-sensitive during the holiday season.   For the magnitude of this effect, consider a soft drink product priced initially at \$1.00 with no promotion.  The median parameter estimates for Specification (C) suggest that during the holiday period, households' willingness-to-pay for this product falls as much as if the price for the product increases by \$0.27 during non-holiday periods.\footnote{$-0.77\alpha = 0.0661 -(0.77+0.36)\alpha(1+0.27)$, where $\alpha=-0.1161$ equals a scaling factor we used to scale the price data so that the price vector has the same length as the $bonus$ vector.  (The rescaling of data vectors is without loss of generality, and improves the performance of random projection by Eq. (\ref{varrp}).)}

We also obtain a positive sign on {\em bonus}, and the negative sign on the interaction {\em price $\times$ bonus} across all specifications, although their magnitudes are small, and there is more variability in these parameters across the different random projections.   We see that discounts seem to make consumers more price sensitive (ie. make the price coefficient more negative).  
 Since any price discounts will be captured in the {\em price} variable itself, the {\em bonus} coefficients capture additional effects that the availability of discounts has on behavior, beyond price.  Hence, the negative coefficient on the interaction {\em price $\times$ bonus} may be consistent with a bounded-rationality view of consumer behavior, whereby the availability of discount on a brand draws consumers' attention to its price, making them more aware of a product's exact price once they are aware that it is on sale.

In specification (D), we introduce the store-level covariate {\em medium-tier}, interacted with {\em price}.   However, the estimates of its coefficient are noisy, and vary widely across the 100 random projections.   This is not surprising, as {\em medium-tier} is a time-invariant variable and, apparently here, interacting it with price still does not result in enough variation for reliable estimation.

\section{Conclusion}
In this paper, we used of random projection -- an important tool for dimension-reduction from machine learning -- for estimating multinomial-choice models with large choice sets, a model which arises in many empirical applications.   Unlike many recent applications of machine learning in econometrics, dimension-reduction here is not required for selecting among high-dimensional covariates, but rather for reducing the inherent high-dimensionality of the model (ie. reducing the size of agents' choice sets).   

Our estimation procedure takes two steps.   First, the high-dimensional choice data are projected (embedded stochastically) into a lower-dimensional Euclidean space. This procedure is justified via results in machine learning, which shows that the pairwise distances between data points are preserved during data compression.  As we show, in practice the random projection can be very sparse, in the sense that only a small fraction (1\%) of the dataset is used in constructing the projection.   In the second step, estimation proceeds using the cyclic monotonicity inequalities implied by the multinomial choice model.   By using these inequalities for estimation, we avoid making explicit distributional assumptions regarding the random utility errors; hence, our estimator is semi-parametric.   The estimator works well in computational simulations and in an application to a real-world supermarket scanner dataset.

We are currently considering several extensions.   First, we are undertaking another empirical application in which consumers can choose among bundles of brands, which would thoroughly leverage the benefits of our random projection approach.   Second, another benefit of random projection is that it preserves privacy, in that the researcher no longer needs to handle the original dataset but rather a ``jumbled-up'' random version of it.\footnote{cf. \cite{Heffetz2014privacy}.}   We are currently exploring additional applications of random projection for econometric settings in which privacy may be an issue.

\newpage\small
\appendix

\section{Additional Tables and Figures}

  \begin{table}[H]
  \small
  \begin{center}
\begin{tabular}{|c||c|c|}
\hline
Design     & min LB, max UB  &  True id set \\\hline\hline
$d=100, k=10$   &  [1.8007, 3.3087]   & [1.4237, 3.2144]   \\

$d=500, k=100$   &   [1.7536, 2.9317]  & [1.2352, 3.4343] \\

$d=1000, k=100$   &  [1.6593, 2.9317]  &[1.1410, 3.4972]  \\ 

$d=5000, k=100$ &  [1.6593, 3.1202] & [1.2038, 3.5914]  \\
  
$d=5000, k=500$   &  [1.6593, 3.1202] & [1.2038, 3.5914]  \\
%
%
%
\hline
\end{tabular}
\caption{\small Random projection estimator with optimal random projections, $s=1$. Replicated 100 times using independently realized random projection matrices. The true value of $\theta$ is 2.3562. Identified set is the interval of points that minimized the unprojected criterion function.}\label{iidapp}
  \end{center}
  \end{table}

  \begin{table}[H]
  \small
  \begin{center}
\begin{tabular}{|c||c|c|}
\hline
Design  &  min LB, max UB & True id set  \\\hline\hline
$d=100, k=10$   &   [1.4237, 2.9788] & [1.4237, 3.2144]   \\

$d=500, k=100$   &   [1.7536, 2.9788]  & [1.2352, 3.4343]   \\

$d=1000, k=100$   &  [1.6122, 3.0259]  & [1.1410, 3.4972]  \\ 

$d=5000, k=100$ & [1.4237, 3.3558]& [1.2038, 3.5914]  \\
  
$d=5000, k=500$   & [1.6593, 3.0259]& [1.2038, 3.5914]    \\

%
\hline
\end{tabular}
\caption{\small Random projection estimator with sparse random projections, $s = \sqrt{d}$. Replicated 100 times using independently realized {\bf sparse} random projection matrices (where $s = \sqrt{d}$ in Definition \ref{sparserp}). The true value of $\theta$ is 2.3562. Identified set is the interval of points that minimized the unprojected criterion function.}
\label{sparseapp}
  \end{center}
  \end{table}

\section{Equivalence of alternative representation of cyclic monotonicity}
Here we show the equivalence of Eqs. (\ref{cm}) and (\ref{cm_2}), as two alternative statements of the cyclic monotonicity inequalities.   We begin with the second statement (\ref{cm_2}).  We have
$$\sum_{l=2}^{L+1}\|\bm{u}^l-\bm{p}^{l} \|^2=\sum_{l=2}^{L+1} \sum_{j=1}^d \left(u_{j}^{l}-p_j^{l}\right)^2 = \sum_{l=2}^{L+1}\left[ \sum_{j=1}^{d} (u_j^{l})^2 + \sum_{j=1}^{d} (p_j^{l})^2 -2\sum_{j=1}^{d} u_j^{l} p_j^{l}\right].$$
Similarly
$$\sum_{l=2}^{L+1}\|\bm{u}^{l}-\bm{p}^{l-1}\|^2=\sum_{l=2}^{L+1} \sum_{j=1}^d \left(u_{j}^{l}-p_{j}^{l-1}\right)^2 = \sum_{l=2}^{L+1}\left[ \sum_{j=1}^{d} (u_j^{l})^2 + \sum_{j=1}^{d} (p_j^{l-1})^2 -2\sum_{j=1}^{d} u_j^{l} p_{j}^{l-1}\right].$$
In the previous two displayed equations, the first two terms cancel out.   By shifting the $l$ indices forward we have:
$$\sum_{l=2}^{L+1} \sum_{j=1}^{d} u_j^{l} p_{j}^{l-1} = \sum_{l=1}^{L}\sum_{j=1}^{d} u_{j}^{l+1}p_j^{l}.$$
Moreover, by definition of a cycle that  $u_j^{L+1} = u_{j}^{1}$, $p_j^{L+1} = p_{j}^{1}$, we then have:
$$\sum_{l=2}^{L+1}\sum_{j=1}^{d} u_j^{l} p_j^{l} = \sum_{l=1}^{L}\sum_{j=1}^{d} u_j^{l} p_j^{l}$$ 

Hence $$\sum_{l=2}^{L+1} \left(\|\bm{u}^{l}-\bm{p}^{l}\|^2-\|\bm{u}^{l}-\bm{p}^{l-1}\|^2 \right)= 2\sum_{l=1}^{L} \sum_{j=1}^{d} \left(u_{j}^{l}p_j^{l-1} - u_j^l p_j^l\right)= 
2\sum_{l=1}^{L} (\bm{u}^{l+1} - \bm{u}^l)\cdot \bm{p}^l$$

Therefore, cyclic monotonicity of Eq. (\ref{cm}) is satisfied if and only if this formulation of cyclic monotonicity in terms of Euclidean norms is satisfied.
\hfill

\qed

\section{Proof of Theorems in Section \ref{convsection}}

We first introduce two auxiliary lemmas.

\begin{lemma}[Convexity Lemma, \cite{pollard1991asymptotics}]
Suppose $A_n(s)$ is a sequence of convex random functions defined on an open convex set $S$ in $\mathbb{R}^{d}$, which converges in probability to some $A(s)$, for each $s$. Then $\sup_{s \in K} |A_{n}(s) - A(s)|$ goes to zero in probability, for each compact subset $K$ of $S$.
\end{lemma}

\begin{lemma}
The criterion function $Q(\bm{\beta},\mathcal{D})$ is convex in $\bm{\beta} \in \mathbb{B}$ for any given dataset $\mathcal{D}$, where $\mathbb{B}$ is an open convex subset of $\mathbb{R}^{b}$.\label{convexQ}
\end{lemma}

\begin{proof}
We want to show that $Q(\lambda\bm{\beta}+(1-\lambda)\bm{\beta}') \leq \lambda Q(\bm{\beta})+(1-\lambda)Q(\bm{\beta}')$, where $\lambda \in [0,1]$, and we suppress the dependence of $Q$ on the data $\mathcal{D}$. 

\begin{align}
&Q(\lambda\bm{\beta} +(1-\lambda)\bm{\beta}' ) \notag \\
=& \sum_{\text{all cycles in data } \mathcal{D}}\Bigg[\sum_{l=1}^{L} \left( \bm{X}^{(a_{l+1})} - \bm{X}^{(a_{l})}\right)\left(\lambda \bm{\beta} + (1-\lambda)\bm{\beta}' \right) \cdot {\bm p}^{(a_{l})}\Bigg]^2_{+} \notag\\
=&\sum_{\text{all cycles in data } \mathcal{D}}\Bigg[\lambda\sum_{l=1}^{L} \left( \bm{X}^{(a_{l+1})} - \bm{X}^{(a_{l})}\right)\bm{\beta}  \cdot {\bm p}^{(a_{l})}+ (1-\lambda)\sum_{l=1}^{L} \left( \bm{X}^{(a_{l+1})} - \bm{X}^{(a_{l})}\right)\bm{\beta}'\cdot {\bm p}^{(a_{l})}\Bigg]^2_{+} \notag\\
\leq&\sum_{\text{all cycles in data } \mathcal{D}}\Bigg\{\lambda\Bigg[\sum_{l=1}^{L} \left( \bm{X}^{(a_{l+1})} - \bm{X}^{(a_{l})}\right)\bm{\beta}  \cdot {\bm p}^{(a_{l})}\Bigg]_{+}+ (1-\lambda)\Bigg[\sum_{l=1}^{L} \left( \bm{X}^{(a_{l+1})} - \bm{X}^{(a_{l})}\right)\bm{\beta}'\cdot {\bm p}^{(a_{l})}\Bigg]_{+}\Bigg\}^2 \label{ineqmax1}\\
\leq& \lambda\sum_{\text{all cycles in data } \mathcal{D}}\Bigg[\sum_{l=1}^{L} \left( \bm{X}^{(a_{l+1})} - \bm{X}^{(a_{l})}\right)\bm{\beta}  \cdot {\bm p}^{(a_{l})}\Bigg]^2_{+} + \notag\\
&\quad (1-\lambda)\sum_{\text{all cycles in data } \mathcal{D}}\Bigg[\sum_{l=1}^{L} \left( \bm{X}^{(a_{l+1})} - \bm{X}^{(a_{l})}\right)\bm{\beta}'\cdot {\bm p}^{(a_{l})}\Bigg]^2_{+} \label{ineqmax2}\\
=&\lambda Q(\bm{\beta})+(1-\lambda)Q(\bm{\beta}')\notag
\end{align} 

Inequality \ref{ineqmax1} above is due to the fact that $\max\{x,0\} + \max\{y,0\} \geq \max\{x+y,0\}$ for all $x,y \in \mathbb{R}$. Inequality \ref{ineqmax2} holds from the convexity of the function $f(x)=x^2$.
\end{proof}

{\bfseries Proof of Theorem \ref{ucon}:} Recall from Eq. (\ref{varrp})  that for any two vectors ${\bm u}, {\bm v} \in \mathbb{R}^{d}$, and for the class of $k$-by-$d$ random projection matrices, $R$, considered in Definition \ref{sparserp}, we have:

\begin{align}
\E(\|R \bm{u}-R\bm{v}\|^{2}) = \|\bm{u}-\bm{v}\|^{2} \\
\Var(\|R \bm{u}-R\bm{v}\|^{2}) = O\left(\frac{1}{k}\right)
\end{align}

Therefore by  Chebyshev's inequality, $\|R \bm{u}-R\bm{v}\|^{2}$ converges in probability to $\|\bm{u}-\bm{v}\|^{2}$ as $k$ grows large. It follows that for any given $\bm{X}$, $\bm{\beta}$ and $\bm{p}$, we have $\|\tilde{\bm{X}}\bm{\beta} - \tilde{\bm{p}}\|^{2} \rightarrow_{p} \|\bm{X}\bm{\beta} - \bm{p}\|^{2}$, where $\tilde{\bm{X}}=R\bm{X}$ and  $\tilde{\bm{p}}=R\bm{p}$ are the projected versions of $\bm{X}$ and  $\bm{p}$. Applying the Continuous Mapping Theorem to the criterion function in Eq. \ref{QEuclidean}, we obtain that $Q(\bm{\beta},\tilde{\mathcal{D}}_{k})$ converges in probability to $Q(\bm{\beta},\mathcal{D})$ pointwise for every $\bm{\beta}$ as $k$ grows large.

By Lemma \ref{convexQ}, the criterion function $Q(\bm{\beta},\mathcal{D})$ is convex in $\bm{\beta} \in \mathbb{B}$ for any given data $\mathcal{D}$, where $\mathbb{B}$ is an open convex subset of $\mathbb{R}^{b}$. Therefore, we can immediately invoke the Convexity Lemma to show that pointwise convergence of the $Q$ function implies that $Q(\bm{\beta},\tilde{\mathcal{D}}_{k})$ converges uniformly to $Q(\bm{\beta},\mathcal{D})$.
\hfill$\square$

{\bfseries Proof of Theorem \ref{convergence}:}
The result follows readily from Assumption \ref{exist} and Theorem \ref{ucon}, by invoking \cite{chernozhukov2007estimation}. The key is in recognizing that (i) in our random finite-sampled criterion function, the randomness stems from the $k$-by-$d$ random projection matrix, (ii) the deterministic limiting criterion function here is defined to be the criterion function computed without random projection, taking the full dataset as given. We can then strengthen the notion of half-Hausdorff convergence to full Hausdorff convergence following the augmented set estimator as in \cite{chernozhukov2007estimation}.
\hfill$\square$

\section{Additional convergence result}\label{add}

\begin{assumption}\label{dgrows}
Suppose that as the dimensionality of the choice set, $d$, grows, the (deterministic) sequence of data $\mathcal{D}_{d}=\{(\bm{X}^{(1)},\bm{p}^{(1)}),$ $\dots,(\bm{X}^{(n)},\bm{p}^{(n)})\}$ satisfies the following two assumptions.
(i) Let $\bm{X}$ be any vector of covariates in $\mathcal{D}_{d}$, then $\frac{1}{d}\sum_{j=1}^{d}(X_{j})^{4}$ exist and is bounded as $d$ grows. Secondly, without loss of generality, assume that for all vectors of covariates $\bm{X}$ in the data $\mathcal{D}_{d}$, $\sum_{j=1}^{d} X_{j}^{2}=\|\bm{X}\|^{2} = O(1)$ as $d$ grows. This part is without loss of generality as the cardinality of utilities can be rescaled. 
\end{assumption}
%
As before, the only source of randomness is in the random projection. Accordingly, the sequence of data $\mathcal{D}_{d}$ as $d$ grows is deterministic.

\begin{theorem}
Suppose that as $d$ grows, the sequence of data $\mathcal{D}_{d}$  satisfies Assumptions \ref{exist} and \ref{dgrows}. Let $R$ be a $k \times d$ sparse random projection matrix with $d = O(k^{2})$, and denote $\tilde{\mathcal{D}}_{k}=R \mathcal{D}_{d}$ as the compressed data. The random projection estimator $\tilde{\Theta}_{k} =\argmin_{{\bm \beta} \in \mathbb{B}} Q({\bm \beta},\tilde{\mathcal{D}}_{k})$ converges in half-Hausdorff distance to the identified set $\Theta^{*}$ as $k$ grows, i.e. $\sup_{{\bm \beta} \in \tilde{\Theta}_{k}} \inf_{{\bm \beta}' \in \Theta^{*}} \|{\bm \beta}  - {\bm \beta}'\|  \xrightarrow[]{p} 0$ as $k$ grows.
\end{theorem}

\begin{proof}
Let $\bm{u}^{(i)}\equiv \bm{X}^{(i)}\bm{\beta}$ be the $d$-dimensional vector of utilities that market $i$ derives from each of the $d$ choices (before realization of shocks), and let $\bm{p}^{(i)}$ be the corresponding observed choice probabilities for market $i$ in the data $\mathcal{D}_d$. For any $\bm{\beta}$, and any pair of markets $(a,b) \in \{1,\dots,n\}^{2}$, we have from Equation \ref{varrp}:

\begin{align}
\Var(\|\tilde{\bm{u}}^{(a)}-\tilde{\bm{p}}^{(b)}\|^{2}) = \frac{1}{\sqrt{d}} \Big(2\|\bm{u}^{(a)}-\bm{p}^{(b)}\|^{4}+(s-3)\sum_{j=1}^{d}(u^{(a)}_{j} - p^{(b)}_{j})^{4} \Big)\label{varproof}
\end{align}

where $(\tilde{\bm{u}}^{(a)},\tilde{\bm{p}}^{(b)})$ denotes the projected-down vector $(R\bm{u}^{(a)}, R \bm{p}^{(b)})$, and $R$ is a $k \times d$ sparse random projection matrix (i.e. $s = \sqrt{d}$). Following \cite{li2006very}, if the limit of $\frac{1}{d}\sum_{j=1}^{d}(u_{j}^{(a)})^{4}$ exists and is bounded as $d$ grows, then the first term in Equation \ref{varproof} dominates the second term. Now, we have $\frac{1}{d}\sum_{j=1}^{d}(u_{j}^{(a)})^{4}  = \frac{1}{d}\sum_{j=1}^{d}\left(\sum_{t=1}^{p} \beta_{t} X^{(a)}_{j,t} \right)^{4}$, where $|\beta_{t}|< 1$ per the normalization of parameters: $\|\bm{\beta}\| = 1$. A sufficient condition for $\frac{1}{d}\sum_{j=1}^{d}\left(\sum_{t=1}^{p} \beta_{t} X^{(a)}_{j,t} \right)^{4}$ to exist in the limit is that $\frac{1}{d}\sum_{j=1}^{d} (X^{(a)}_{j,t})^{4}$ exists for all $t$, as stipulated in Assumption (2). (By the Jensen's inequality, if the fourth moment exists, then all lower moments exist, and by the Cauchy--Schwarz inequality, if $\E[X^{4}]$ and $\E[Y^{4}]$ exist, then $\E[X^{2}Y^{2}]$, $\E[XY^{2}]$ and so on also exist.)

Having established that the first term in Equation \ref{varproof} dominates, we now examine the first term. If $\|\bm{u}^{(a)}\| = O(1)$, then  $\|\bm{u}^{(a)}-\bm{p}^{(b)}\|^{4} = O(1)$ since $\bm{p}^{(b)}$ is a vector of choice probabilities and $\|\bm{p}^{(b)}\|$ is bounded for all $d$. A sufficient condition for  $\|\bm{u}^{(a)}\| \equiv \|\bm{X}^{(a)}\bm{\beta}\| = O(1)$ is that for all columns $\bm{X}$ of $\bm{X}^{(a)}$, we have $\|\bm{X}\|^{2} = O(1)$ as $d$ grows. This is maintained by Assumption \ref{dgrows}. Therefore from Equation \ref{varproof}, we have $\Var(\|\tilde{\bm{u}}^{(a)}-\tilde{\bm{p}}^{(b)}\|^{2})  = O\left(\frac{1}{\sqrt{d}}\right)$ as $d$ grows. Hence the criterion function $Q(\bm{\beta},\tilde{\mathcal{D}}_{k})$ converges pointwise to $Q(\bm{\beta},\mathcal{D})$ as $d$ grows. The rest of the proof follows from  Theorems \ref{ucon} and \ref{convergence}. 

\end{proof}

\newpage\small
\bibliographystyle{apalike}
\bibliography{rpmc}

\begin{thebibliography}{}

\bibitem[Achlioptas, 2003]{achlioptas2003database}
Achlioptas, D. (2003).
\newblock Database-friendly random projections: Johnson-lindenstrauss with
  binary coins.
\newblock {\em Journal of computer and System Sciences}, 66(4):671--687.

\bibitem[Belloni et~al., 2012]{belloni2012sparse}
Belloni, A., Chen, D., Chernozhukov, V., and Hansen, C. (2012).
\newblock Sparse models and methods for optimal instruments with an application
  to eminent domain.
\newblock {\em Econometrica}, 80(6):2369--2429.

\bibitem[Belloni et~al., 2014]{belloni2014high}
Belloni, A., Chernozhukov, V., and Hansen, C. (2014).
\newblock High-dimensional methods and inference on structural and treatment
  effects.
\newblock {\em Journal of Economic Perspectives}, 28(2):29--50.

\bibitem[Berry and Haile, 2014]{berry2014identification}
Berry, S.~T. and Haile, P.~A. (2014).
\newblock Identification in differentiated products markets using market level
  data.
\newblock {\em Econometrica}, 82(5):1749--1797.

\bibitem[Chernozhukov et~al., 2007]{chernozhukov2007estimation}
Chernozhukov, V., Hong, H., and Tamer, E. (2007).
\newblock Estimation and confidence regions for parameter sets in econometric
  models1.
\newblock {\em Econometrica}, 75(5):1243--1284.

\bibitem[Chevalier et~al., 2003]{chevalier2003don}
Chevalier, J.~A., Kashyap, A.~K., and Rossi, P.~E. (2003).
\newblock Why don't prices rise during periods of peak demand? evidence from
  scanner data.
\newblock {\em American Economic Review}, 93(1):15--37.

\bibitem[Chiong et~al., 2016]{chiong2016duality}
Chiong, K., Galichon, A., and Shum, M. (2016).
\newblock Duality in dynamic discrete choice models.
\newblock {\em Quantitative Economics}, 7:83--115.

\bibitem[Coibion et~al., 2015]{coibion2015cyclicality}
Coibion, O., Gorodnichenko, Y., and Hong, G.~H. (2015).
\newblock The cyclicality of sales, regular and effective prices: Business
  cycle and policy implications.
\newblock {\em American Economic Review}, 105(3):993--1029.

\bibitem[Dasgupta and Gupta, 2003]{dasgupta2003elementary}
Dasgupta, S. and Gupta, A. (2003).
\newblock An elementary proof of a theorem of johnson and lindenstrauss.
\newblock {\em Random Structures \& Algorithms}, 22(1):60--65.

\bibitem[Davis et~al., 2016]{davis2016segregated}
Davis, D.~R., Dingel, J.~I., Monras, J., and Morales, E. (2016).
\newblock How segregated is urban consumption?
\newblock Technical report, Columbia University.

\bibitem[Fox, 2007]{fox2007semiparametric}
Fox, J.~T. (2007).
\newblock Semiparametric estimation of multinomial discrete-choice models using
  a subset of choices.
\newblock {\em RAND Journal of Economics}, pages 1002--1019.

\bibitem[Fox and Bajari, 2013]{fox2013measuring}
Fox, J.~T. and Bajari, P. (2013).
\newblock Measuring the efficiency of an fcc spectrum auction.
\newblock {\em American Economic Journal: Microeconomics}, 5(1):100--146.

\bibitem[Gandhi et~al., 2013]{gandhi2013estimating}
Gandhi, A., Lu, Z., and Shi, X. (2013).
\newblock Estimating demand for differentiated products with error in market
  shares.
\newblock Technical report, University of Wisconsin-Madison.

\bibitem[Gentzkow et~al., 2016]{gentzkow2016measuring}
Gentzkow, M., Shapiro, J., and Taddy, M. (2016).
\newblock Measuring polarization in high-dimensional data: Method and
  application to congressional speech.
\newblock Technical report, Stanford University.

\bibitem[Gillen et~al., 2015]{gillen2015blp}
Gillen, B.~J., Montero, S., Moon, H.~R., and Shum, M. (2015).
\newblock Blp-lasso for aggregate discrete choice models of elections with rich
  demographic covariates.
\newblock {\em USC-INET Research Paper}, (15-27).

\bibitem[Goeree et~al., 2005]{goeree2005regular}
Goeree, J.~K., Holt, C.~A., and Palfrey, T.~R. (2005).
\newblock Regular quantal response equilibrium.
\newblock {\em Experimental Economics}, 8(4):347--367.

\bibitem[Haile et~al., 2008]{haile2008empirical}
Haile, P.~A., Horta{\c{c}}su, A., and Kosenok, G. (2008).
\newblock On the empirical content of quantal response equilibrium.
\newblock {\em American Economic Review}, 98(1):180--200.

\bibitem[Han, 1987]{han1987non}
Han, A.~K. (1987).
\newblock Non-parametric analysis of a generalized regression model: the
  maximum rank correlation estimator.
\newblock {\em Journal of Econometrics}, 35(2):303--316.

\bibitem[Hausman and McFadden, 1984]{hausman1984specification}
Hausman, J. and McFadden, D. (1984).
\newblock Specification tests for the multinomial logit model.
\newblock {\em Econometrica}, pages 1219--1240.

\bibitem[Hausman et~al., 1998]{hausman1998misclassification}
Hausman, J.~A., Abrevaya, J., and Scott-Morton, F.~M. (1998).
\newblock Misclassification of the dependent variable in a discrete-response
  setting.
\newblock {\em Journal of Econometrics}, 87(2):239--269.

\bibitem[Heffetz and Ligett, 2014]{Heffetz2014privacy}
Heffetz, O. and Ligett, K. (2014).
\newblock Privacy and data-based research.
\newblock {\em Journal of Economic Perspectives}, 28(2):75--98.

\bibitem[Ichimura and Lee, 1991]{ichimura1991semiparametric}
Ichimura, H. and Lee, L.-F. (1991).
\newblock Semiparametric least squares estimation of multiple index models:
  single equation estimation.
\newblock In {\em Nonparametric and semiparametric methods in econometrics and
  statistics: Proceedings of the Fifth International Symposium in Economic
  Theory and Econometrics. Cambridge}, pages 3--49.

\bibitem[Keane and Wasi, 2012]{keane2012estimation}
Keane, M. and Wasi, N. (2012).
\newblock Estimation of discrete choice models with many alternatives using
  random subsets of the full choice set: With an application to demand for
  frozen pizza.
\newblock Technical report, Oxford University.

\bibitem[Lee, 1995]{lee1995semiparametric}
Lee, L.-F. (1995).
\newblock Semiparametric maximum likelihood estimation of polychotomous and
  sequential choice models.
\newblock {\em Journal of Econometrics}, 65(2):381--428.

\bibitem[Li et~al., 2006]{li2006very}
Li, P., Hastie, T.~J., and Church, K.~W. (2006).
\newblock Very sparse random projections.
\newblock In {\em Proceedings of the 12th ACM SIGKDD international conference
  on Knowledge discovery and data mining}, pages 287--296. ACM.

\bibitem[Manski, 1975]{manski1975maximum}
Manski, C.~F. (1975).
\newblock Maximum score estimation of the stochastic utility model of choice.
\newblock {\em Journal of econometrics}, 3(3):205--228.

\bibitem[Manski, 1985]{manski1985semiparametric}
Manski, C.~F. (1985).
\newblock Semiparametric analysis of discrete response: Asymptotic properties
  of the maximum score estimator.
\newblock {\em Journal of econometrics}, 27(3):313--333.

\bibitem[McFadden, 1974]{mcfadden1974conditional}
McFadden, D. (1974).
\newblock Conditional logit analysis of qualitative choice behaviour.
\newblock In {\em Frontiers in Econometrics}, ed. P. Zarembka.(New York:
  Academic Press).

\bibitem[McFadden, 1978]{mcfadden1978modelling}
McFadden, D. (1978).
\newblock Modelling the choice of residential location.
\newblock Technical report, Institute of Transportation Studies, University of
  California-Berkeley.

\bibitem[McFadden, 1981]{mcfadden1981econometric}
McFadden, D. (1981).
\newblock Econometric models of probabilistic choice.
\newblock In Manski, C. and McFadden, D., editors, {\em Structural Analysis of
  Discrete Data with Econometric Applications}. MIT Press.

\bibitem[Melo et~al., 2015]{melo2015testing}
Melo, E., Pogorelskiy, K., and Shum, M. (2015).
\newblock Testing the quantal response hypothesis.
\newblock Technical report, California Institute of Technology.

\bibitem[Ng, 2015]{ng2015opportunities}
Ng, S. (2015).
\newblock Opportunities and challenges: Lessons from analyzing terabytes of
  scanner data.
\newblock Technical report, Columbia University.

\bibitem[Pollard, 1991]{pollard1991asymptotics}
Pollard, D. (1991).
\newblock Asymptotics for least absolute deviation regression estimators.
\newblock {\em Econometric Theory}, 7(02):186--199.

\bibitem[Powell and Ruud, 2008]{powell2008simple}
Powell, J.~L. and Ruud, P.~A. (2008).
\newblock Simple estimators for semiparametric multinomial choice models.
\newblock Technical report, University of California, Berkeley.

\bibitem[Rockafellar, 1970]{rockafellar1970convex}
Rockafellar, R.~T. (1970).
\newblock {\em Convex analysis}.
\newblock Princeton university press.

\bibitem[Shi et~al., 2016]{shi2016estimating}
Shi, X., Shum, M., and Song, W. (2016).
\newblock Estimating semi-parametric panel multinomial choice models using
  cyclic monotonicity.
\newblock Technical report, University of Wisconsin-Madison.

\bibitem[Train et~al., 1987]{train1987demand}
Train, K.~E., McFadden, D.~L., and Ben-Akiva, M. (1987).
\newblock The demand for local telephone service: A fully discrete model of
  residential calling patterns and service choices.
\newblock {\em RAND Journal of Economics}, pages 109--123.

\bibitem[Vempala, 2000]{vempala2000random}
Vempala, S. (2000).
\newblock {\em The Random Projection Method}.
\newblock American Mathematical Society.
\newblock Series in Discrete Mathematics and Theoretical Computer Science
  (DIMACS), Vol. 65.

\end{thebibliography}

\end{document}